\documentclass[11pt]{article}
\usepackage{amsmath, amssymb, amsthm}
\usepackage{graphicx}
\usepackage{authblk}
\usepackage{times}
\usepackage{hyperref}

\theoremstyle{plain}
\newtheorem{theorem}{Theorem}
\newtheorem{proposition}[theorem]{Proposition}
\newtheorem{lemma}[theorem]{Lemma}
\newtheorem{corollary}{Corollary}
\theoremstyle{definition}

\newtheorem{condition}[theorem]{Condition}

\theoremstyle{remark}

\title{\textbf{A Generalized Spectral Framework to Expain Neural Scaling and Compression Dynamics}}

\author[1]{Yizhou Zhang}
\affil[1]{zyizhou96@gmail.com}

\date{}

\begin{document}
\maketitle

\begin{abstract}
Empirical scaling laws describe how test loss and other performance metrics depend on model size, dataset size, and compute.
While such laws are consistent within specific regimes, apparently distinct scaling behaviors have been reported for related settings such as model compression.
Motivated by recent progress in spectral analyses of neural representations, this paper develops a \emph{generalized spectral framework} that unifies learning dynamics and compression phenomena under a common functional ansatz.
We generalize the spectral evolution function from the linear kernel form $g(\lambda t)=\lambda t$ to an asymptotically polynomial function $g(\lambda,t;\beta)$, characterized by an effective spectral--temporal elasticity $\rho(\beta)$.
This framework recovers existing lazy and feature-learning theories as special cases and yields an invariant relation between learning and compression exponents.
\end{abstract}
\section{Introduction}

\subsection{Motivation}

Scaling laws have emerged as a pervasive empirical regularity in deep learning, 
linking performance metrics such as loss or accuracy to model size, data, and compute \cite{kaplan2020scaling,hoffmann2022training}.
Yet as more phenomena are examined, the landscape of scaling laws has become fragmented:
even conceptually similar processes often follow distinct functional forms.

Model compression illustrates this tension clearly.
Both pruning and quantization reduce the effective capacity of a network,
but their empirical scaling behaviors differ qualitatively.
Pruning exhibits a characteristic three-phase curve---a low-error plateau, 
a power-law region where error increases roughly as a function of sparsity, 
and a high-error collapse at extreme pruning \cite{rosenfeld2021predictability}.
Quantization, by contrast, shows a monotonic and predictable error growth with data and model scale:
at fixed precision, larger models and datasets yield proportionally higher quantization-induced loss,
a relation that can be stably extrapolated across precisions \cite{kumar2024scaling}.
These divergent laws suggest that compression interacts with the spectrum of learned representations in different ways,
yet a unified theoretical explanation remains missing.

\subsection{Spectral perspectives on learning}

Spectral analyses have provided powerful tools for explaining learning dynamics in overparameterized networks.
In the \emph{lazy} or kernel regime, each eigenmode evolves as
\[
f_k(t)=w_k(1-e^{-\lambda_k t}),
\]
yielding an exponential relaxation of modes and a simple power-law decay of loss over time.
This form arises naturally in the neural tangent kernel (NTK) limit, where the feature covariance remains fixed throughout training \cite{jacot2018ntk,lee2019wide}.

Beyond this regime, several works have demonstrated that spectral evolution in realistic networks is generally
\emph{non-homogeneous} in both time and eigenvalue.
In the feature-learning limit, Bordelon and Pehlevan~\cite{bordelon2024feature} showed that
\[
f_k(t)=w_k\bigl(1-e^{-\lambda_k\,t^{p(\beta)}}\bigr),
\qquad
p(\beta)=\max\!\left(1,\frac{2}{1+\beta}\right),
\]
introducing a task-dependent temporal exponent $p(\beta)$ that modifies how spectral energy accumulates.
Yang~\cite{yang2021tensor} further established, through the Tensor Programs framework,
that the infinite-width limit depends on parameterization:
different scaling rules lead to either a static NTK or an evolving, feature-learning kernel.
Recent exact analyses in deep linear networks~\cite{domine2024lazy2rich}
and finite-width experiments~\cite{wang2023spectralevolution,vyas2023featurelearning}
reveal similar transitions between lazy and rich regimes,
where spectral modes evolve with depth, initialization, and learning rate.
Kernel-level studies~\cite{canatar2022kernelfeature} likewise show that the effective kernel itself drifts
during training, adapting to the data distribution.

Together, these results indicate that there is no unique spectral evolution law:
different architectures, parameterizations, and training settings induce distinct dynamics.
To reason about their shared structure, we introduce a \emph{generalized spectral evolution function}
$g(\lambda,t;\beta)$ that encompasses these cases and enables analysis of loss decay,
model density, and spectral robustness within a single mathematical framework.

\subsection{Our perspective}

In this work, we abstract over these diverse spectral dynamics and propose a 
\emph{generalized spectral evolution function}
\[
f_k(t)=w_k\bigl(1-e^{-g(\lambda_k,t;\beta)}\bigr),
\]
where $g(\lambda,t;\beta)$ is assumed only to be smooth, monotone, and asymptotically polynomial in both $\lambda$ and $t$.
All known spectral dynamics---kernel, feature-learning, or hybrid---can be expressed within this form.

This abstraction allows us to analyze learning, compression, and spectral robustness
under a single mathematical framework.
Specifically:
\begin{itemize}
    \item We show that any $g$ satisfying mild regularity conditions naturally yields 
    a power-law decay of test loss over training time,
    with an exponent determined by its \emph{log–elasticity ratio}
    \[
    \rho(\beta)
    = -\frac{\partial_{\log t}\log g}{\partial_{\log \lambda}\log g}.
    \]
    \item We extend this framework to characterize model robustness to 
    spectral perturbations, deriving the induced excess loss $\Delta L(t)$
    as a function of $\rho(\beta)$.
    \item Using this $\Delta L(t)$ analysis, we interpret pruning and quantization
    as two different forms of perturbing the learned spectrum:
    pruning truncates the spectral tail, while quantization perturbs
    modes near the learning frontier.
    Both emerge as special cases of the same generalized theory,
    explaining their distinct yet predictable empirical scaling behaviors.
    \item Our framework also predicts that the \emph{model density}---the number of effectively learned spectral modes per parameter---follows a power-law in available compute. Consequently, as compute increases exponentially over time, model density also grows exponentially, implying that sufficiently trained smaller models can spectrally match or even surpass undertrained larger models. This theoretical ``spectral densing'' effect is consistent with the empirical \emph{Densing Law} recently observed by Xiao~et~al.~\cite{Xiao2024DensingLaw}.
\end{itemize}

In summary, instead of proposing another specific spectral model,
we provide a general theoretical template that captures
how loss, model density, and compressibility co-evolve under a wide class of spectral dynamics.

\section{Related Work}

\subsection{Scaling laws in deep learning}

Scaling laws describe how performance metrics vary predictably with model size,
dataset size, or compute resources.
Empirical studies have shown that test loss often follows a simple power law
with respect to these quantities \cite{kaplan2020scaling,hoffmann2022training},
and similar regularities appear across modalities and architectures.
Beyond this primary scaling, a growing body of work has identified secondary
scaling laws governing compression, precision, and sparsity.
For instance, pruning experiments reveal structured error–density curves with
plateaus and power-law regions \cite{rosenfeld2021predictability},
while quantization exhibits smooth, predictable degradation with limited
precision \cite{kumar2024scaling}.
These phenomena suggest that both learning and compression are governed by
underlying low-dimensional dynamics, motivating theoretical frameworks that can
unify their behavior.

\subsection{Spectral theories of learning dynamics}

A central approach to explaining these regularities is through spectral
analysis of training dynamics.
In the \emph{lazy} regime, networks behave as linear models with fixed features,
and each eigenmode relaxes exponentially as
$f_k(t)=w_k(1-e^{-\lambda_k t})$ \cite{jacot2018ntk,lee2019wide}.
This yields analytic control and captures the early, kernel-dominated phase of
training.

Recent studies extend this view to regimes where the feature covariance evolves
during training.
Bordelon and Pehlevan~\cite{bordelon2024feature} showed that the temporal
exponent of spectral relaxation depends on task smoothness,
leading to $f_k(t)=w_k(1-e^{-\lambda_k t^{p(\beta)}})$ with
$p(\beta)=\max(1,2/(1+\beta))$.
Yang~\cite{yang2021tensor} formalized this transition using the Tensor Programs
framework, demonstrating that different parameterizations yield distinct limits
---either a fixed NTK or a feature-learning kernel.
Exact analyses in deep linear networks
\cite{domine2024lazy2rich} and empirical studies in finite-width settings
\cite{wang2023spectralevolution,vyas2023featurelearning}
confirm that networks interpolate continuously between these two regimes.
Further, kernel-level analyses reveal that the effective kernel itself drifts
during training, adapting to data and initialization
\cite{canatar2022kernelfeature}.
These results collectively indicate that spectral evolution is not universal,
but depends on architectural and dynamical choices.
Our work builds upon these insights by abstracting their shared structure into a
generalized spectral function $g(\lambda,t;\beta)$, which provides a unified
language for describing learning, loss decay, and spectral robustness.

\subsection{Compression and spectral robustness}

A parallel line of research studies compression phenomena such as pruning and
quantization.
Early works demonstrated that large networks can be drastically pruned without
loss of accuracy, revealing redundancy and spectral sparsity
\cite{han2016deep,frantar2022gptq}.
Rosenfeld \emph{et al.}~\cite{rosenfeld2021predictability} systematically
analyzed scaling of test error with sparsity and model size, discovering a
three-phase error–density curve characterized by a low-error plateau, a
power-law region, and a high-error plateau, consistent across architectures.
For quantization, Kumar \emph{et al.}~\cite{kumar2024scaling} introduced
precision-aware scaling laws, showing that quantization error grows smoothly
and predictably with model scale and dataset size.
Despite their differing empirical forms, both processes can be interpreted as
perturbations of the learned spectral distribution:
pruning truncates the tail of the spectrum, whereas quantization perturbs modes
near the learning frontier.
The present work formalizes this view by deriving loss and compression scaling
from a common spectral function $g(\lambda,t;\beta)$,
linking model capacity, density evolution, and compression robustness within a
single theoretical framework.

\section{Generalized Spectral Theory}

\subsection{Spectral dynamics and the generalized evolution function}

We model learning along spectral modes through the generalized evolution
\begin{equation}
f_\lambda(t) = w_\lambda \bigl(1 - e^{-g(\lambda,t;\beta)}\bigr),
\label{eq:general_f}
\end{equation}
where $g(\lambda,t;\beta)$ describes the cumulative learning progress of each
eigenmode~$\lambda$. $\beta$ is a parameter of the function, which can be characterized by the learning regime.
The function $g$ is assumed to be smooth, monotonic in both arguments, and to
exhibit an asymptotically monomial form in the joint limit
$\lambda\to 0^+$ and $t\to\infty$. Namely, we have:

\begin{condition}[Asymptotic polynomial behavior of $g$ in the spectral tail]
\label{cond:poly}
We assume that in the joint regime of small eigenvalues and large time,
the spectral dynamics $g(\lambda,t;\beta)$ behaves approximately like a
monomial $C(\beta)\lambda^{a(\beta)} t^{b(\beta)}$. 
Formally, there exists a shrinking cutoff $\lambda_0(t)\to 0$ such that
for sufficiently large $t$,
\begin{equation}
\lim_{t\to\infty}\;\;
\sup_{0<\lambda<\lambda_0(t)}
\Biggl|
\frac{
g(\lambda,t;\beta)
}{
C(\beta)\,\lambda^{a(\beta)}\,t^{b(\beta)}
}
-1
\Biggr|
= 0.
\label{eq:g_growth}
\end{equation}
\end{condition}

Intuitively, Condition~\ref{cond:poly} is mild and is satisfied by a wide class of
spectral learning dynamics, including but not limited to all monomials and finite sums of monomials. The condition excludes only
sub-polynomial growth such as $\log(1+\lambda^a t^b)$.
This form includes the NTK regime
($m(\beta)=n(\beta)=1$), the feature-learning regime of~\cite{bordelon2024feature}
($m(\beta)=1$, $n(\beta)=p(\beta)$), and other settings where
spectral and temporal elasticities differ.
The central idea is that any $g$ satisfying~\eqref{eq:g_growth} induces
a characteristic \emph{learning frontier} separating learned and unlearned modes.

\paragraph{Learning frontier.\label{para:learning_frontier}}
Let the instantaneous learning rate be
\[
R(\lambda,t) = \partial_t f_\lambda(t)
= w_\lambda\,e^{-g(\lambda,t;\beta)}\,\partial_t g(\lambda,t;\beta).
\]
The learning frontier $\lambda_*(t)$ is defined as the mode maximizing this
rate,
\[
\lambda_*(t) =\arg\max_{\lambda} R(\lambda,t),
\]
which equivalently satisfies the stationarity condition
\[
\partial_\lambda\!\ln R(\lambda_*(t),t)
= -\partial_\lambda g(\lambda_*,t;\beta)
 +\frac{\partial_\lambda\partial_t g(\lambda_*,t;\beta)}{\partial_t g(\lambda_*,t;\beta)}=0.
\]
Asymptotically, this condition selects the level set
\begin{equation}
g(\lambda_*(t),t;\beta)=\kappa,
\label{eq:frontier_def}
\end{equation}
for some constant $\kappa>0$, which we use as the operational definition of the
learning frontier. We provides a rigorous justification for the above properties in Appendix \ref{justification_learning_frontier}

\subsection{Log--elasticity and effective spectral scaling}

We define the \emph{log--elasticities} of $g$ with respect to $\lambda$ and $t$:
\[
E_\lambda(g)=\partial_{\log\lambda}\log g,
\qquad
E_t(g)=\partial_{\log t}\log g.
\]
Their ratio,
\begin{equation}
\rho(\beta)
= -\frac{E_t(g)}{E_\lambda(g)}
= -\frac{\partial_{\log t}\log g}{\partial_{\log \lambda}\log g},
\label{eq:rho_def}
\end{equation}
characterizes the relative scaling of spectral and temporal growth.
Intuitively, $\rho(\beta)$ controls how quickly the learning frontier
$\lambda_*(t)$ recedes:
\begin{equation}
\lambda_*(t)\propto t^{-\rho(\beta)}.
\label{eq:frontier_scaling}
\end{equation}
Special cases recover known limits:
$\rho=1$ for the NTK regime ($g=\lambda t$),
$\rho=p(\beta)$ for the feature-learning regime
($g=\lambda t^{p(\beta)}$), and
$\rho=n(\beta)/m(\beta)$ for a general monomial $g\!\sim\!\lambda^{m(\beta)} t^{n(\beta)}$.

\subsection{Loss scaling and spectral dominance}
\label{sec:generalized-spectral-theory}
The test loss is given by the residual energy of unlearned modes,
\begin{equation}
L(t)
= \sum_{k=1}^\infty \lambda_k w_k^2 e^{-2g(\lambda_k,t)}.
\label{eq:loss_def}
\end{equation}
Intuitively, modes with $\lambda_k$ well above the frontier are nearly
saturated, while modes well below it are essentially unlearned; the dominant
contribution comes from modes in a band around and beyond the learning
frontier. Therefore, the loss is asymptotically dominated by the spectral tail beyond the learning frontier. The following proposition formalizes this statement.

\begin{proposition}[Tail dominance of the loss]
\label{prop:tail_dominance}
Under standard regularity assumptions
($\lambda_k\!\sim\!k^{-b}$, $\lambda_k w_k^2\!\sim\!k^{-a}$, $a,b>1$)
and for any $g$ satisfying the growth
condition~\eqref{eq:g_growth}, the loss~\eqref{eq:loss_def} obeys
\[
L(t)\;\asymp\;
\sum_{k>k_*(t)} \lambda_k w_k^2,
\]
where $\lambda_*(t)$ is defined by~\eqref{eq:frontier_def} and
$\lambda_{k_*(t)}\!\asymp\!\lambda_*(t)$.
Hence the asymptotic decay of $L(t)$ is governed by the spectral tail beyond the
frontier.
\end{proposition}

\paragraph{Proof sketch.}
Using the growth bounds on $g$ and the frontier condition
$g(\lambda_{\ast},t)=\kappa$, one shows that
$\lambda_{\ast}(t)\propto t^{-n(\beta)/m(\beta)}$ and
$k_{\ast}(t)\propto t^{n(\beta)/(m(\beta) b)}$.
The bounds on $g$ imply that $g(\lambda_k,t)$ is comparable to
$\bigl(\lambda_k/\lambda_{\ast}(t)\bigr)^{m(\beta)}$ up to constants.
Thus, for $k\le c_1'k_{\ast}$ (modes well above the frontier),
$g(\lambda_k,t)\ge (1+\varepsilon)\kappa$ and their contribution is uniformly
suppressed by a factor $e^{-2(1+\varepsilon)\kappa}$.
For $k\ge c_2'k_{\ast}$ (modes well below the frontier),
$g(\lambda_k,t)\le(1-\varepsilon)\kappa$ and $e^{-2g(\lambda_k,t)}$ stays
bounded away from zero.
This yields the sandwich bound~\eqref{eq:tail_sandwich}, from which the
asymptotic order $L(t)\asymp k_{\ast}^{1-a}$ follows by standard estimates
for power-law tails.
A detailed proof is given in Appendix~\ref{app:tail_proof}.
\hfill$\square$

From the frontier condition $g(\lambda_{\ast}(t),t;\beta)=\kappa$
and the asymptotic growth of $g=\Theta(\lambda^{m(\beta)}t^{n(\beta)})$, we have
\[
\lambda_{\ast}(t) \propto t^{-\rho(\beta)},\qquad
k_{\ast}(t) \propto \lambda_{\ast}(t)^{-1/b} \propto t^{\rho(\beta)/b}.
\]
Using Proposition \eqref{prop:tail_dominance},
\[
L(t)
\;\asymp\;
\sum_{k>k_{\ast}(t)} k^{-a}
\;\asymp\; k_{\ast}(t)^{1-a}
\;\propto\; t^{-\frac{a-1}{b}\rho(\beta)}.
\]
Thus the loss decays as a power law
\begin{equation}
L(t)\;\propto\; t^{-\chi(\beta)},
\qquad
\chi(\beta)=\frac{a-1}{b}\,\rho(\beta).
\label{eq:chi_def}
\end{equation}
This recovers the standard kernel result when $\rho(\beta)=1$ and matches
the generalized exponents in \cite{bordelon2024feature}.

\subsection{Spectral Perturbations and Compression Robustness}
\label{sec:perturbations}

We now consider the combined effect of model compression and spectral
dynamics.  Both pruning and quantization can be viewed as perturbations to
the learned coefficients or basis functions in the spectral domain.
Under the generalized spectral evolution introduced in Section~3.1, these
perturbations induce a predictable growth of loss that obeys a scaling
relation consistent with the underlying spectral elasticity.

%\section{Theorem: Scaling of Perturbation-Induced Loss under Generalized Spectral Dynamics}

\begin{theorem}[Combined Perturbation-Induced Loss Growth under Generalized Spectral Dynamics]
\label{theorem:pertubation}
Consider standard regularity assumptions
($\lambda_k\!\sim\!k^{-b}$, $\lambda_k w_k^2\!\sim\!k^{-a}$, $a,b>1$)
and $g$ satisfying the growth
condition~\eqref{eq:g_growth}.
Suppose that at inference time, both the basis functions and learned coefficients are perturbed:
\begin{align*}
\tilde{f}(x) &= \sum_k \left( f_k(t) + \eta_k \right) \left( \phi_k(x) + \epsilon_k(x) \right), \\
\epsilon_k(x) &\sim \mathcal{N}(0, \sigma^2), \quad \eta_k \sim \mathcal{N}(0, \tau^2), \quad \text{independently.}
\end{align*}
Then, the increase in test loss due to these perturbations is:
\[
\Delta L(t) = \mathbb{E}_x\left[ \left( \tilde{f}(x) - f^*(x) \right)^2 \right] - L(t)
= \Theta\left( \sigma^2 t^{\rho(\beta)(1 - a + b)/b} \right).
\]
\end{theorem}

\paragraph{Proof sketch.}
We interpret pruning and quantization as small, independent perturbations of the
spectral coefficients and basis functions around the learned predictor.
Writing the perturbed model as
\[
\tilde f(x)
=
\sum_k (f_k(t) + \eta_k)(\phi_k(x) + \epsilon_k(x)),
\]
and using orthonormality of $\{\phi_k\}$ together with the zero-mean,
independence assumptions on $(\eta_k,\epsilon_k)$, all cross terms vanish in
expectation. This reduces the excess loss to a sum of diagonal contributions,
\[
\Delta L(t)
~=~
\mathbb{E}_x\big[(\tilde f(x) - f^\star(x))^2\big] - L(t)
~=~
\Theta\!\left(
\sigma^2 \sum_k f_k(t)^2
\right),
\]
where $\sigma^2$ collects the perturbation scales, and the remaining terms
involving $\lambda_k$ alone converge due to $b>1$ and thus do not affect the
asymptotic $t$-dependence.

The generalized spectral dynamics imply a learning frontier $k^\ast(t)$ such
that modes with $k \lesssim k^\ast(t)$ are largely learned ($f_k(t)\approx w_k$),
while modes with $k \gg k^\ast(t)$ remain small ($f_k(t)\approx 0$).
Under the tail assumptions $\lambda_k \sim k^{-b}$ and
$\lambda_k w_k^2 \sim k^{-a}$ with $a,b>1$, we have
$w_k^2 \sim k^{b-a}$ and hence
\[
\sum_k f_k(t)^2
~\approx~
\sum_{k \le k^\ast(t)} w_k^2
~\asymp~
\bigl(k^\ast(t)\bigr)^{1 + b - a}.
\]
From the definition of the frontier and the log–elasticity ratio $\rho(\beta)$,
we obtain $k^\ast(t) \propto t^{\rho(\beta)/b}$, which yields
\[
\Delta L(t)=
\Theta\!\left(
\sigma^2 t^{\rho(\beta)(1 + b - a)/b}
\right),
\]
establishing the claimed power-law exponent.
A rigorous derivation, including precise constants and bounds for all
perturbation terms, is deferred to Appendix \ref{app:proof-theorem-2}.
\hfill$\square$

\begin{corollary}[Complementary Relationship Between Compression Loss and Original Test Loss]
In the generalized spectral regime, define the exponents:
\begin{itemize}
  \item Test loss scaling: \( L(t) \sim t^{-\chi(\beta)}, \quad \chi(\beta) = \tfrac{a - 1}{b}\rho(\beta) \)
  \item Compression-induced loss scaling: \( \Delta L(t) \sim t^{\gamma(\beta)}, \quad \gamma(\beta) = \tfrac{1 - a + b}{b}\rho(\beta) \)
\end{itemize}
Then, these exponents satisfy the universal relation:
\[
\boxed{\chi(\beta) + \gamma(\beta) = \rho(\beta)}.
\]
\end{corollary}

This expression reflects a conserved spectral elasticity, where the rate of loss decay and compression-induced error jointly saturate the same elasticity bound. It implies that hard tasks (i.e. $\chi(\beta)$ is smaller and the loss decays slower as data increases) leads to less robustness to compression.

%This relation will be empirically illustrated in Section~4.

\subsection*{Summary}

Section~3 establishes a unified framework for spectral learning dynamics.
The generalized function $g(\lambda,t;\beta)$ defines a continuum of regimes
between lazy and feature-learning limits.
Its elasticity ratio $\rho(\beta)$ determines both the temporal scaling of the
learning frontier and the asymptotic decay of loss.
Moreover, spectral perturbations inherit the same structure, allowing a unified
treatment of pruning and quantization as spectral responses.

\section{Applications to Compression Phenomena}

The generalized spectral theory developed above provides a unified framework for interpreting diverse forms of model compression.
In this section we apply it to two representative cases—quantization and pruning—which respectively correspond to smooth perturbations and hard truncations of the learned spectral energy distribution.
\subsection{Quantization as Joint Coefficient and Feature Perturbation}

We now interpret post-training or in-training quantization as a concrete instance of the generalized spectral perturbation model introduced in Section \ref{sec:perturbations}.
Recall that the learned predictor admits the expansion
\begin{equation}
    f(x, t) = \sum_k f_k(t) \, \phi_k(x), 
    \qquad 
    f_k(t) = w_k \bigl(1 - e^{-g(\lambda_k, t; \beta)}\bigr),
\end{equation}
where $\{\phi_k\}$ is an orthonormal basis aligned with the data distribution, and $g(\lambda, t; \beta)$ satisfies the generalized growth and elasticity conditions of Section~\ref{sec:generalized-spectral-theory}.

Under finite-precision constraints, quantization perturbs both the coefficients and the effective features. We write the quantized model as
\begin{equation}
    \tilde{f}(x, t) 
    = \sum_k \bigl(f_k(t) + \eta_k\bigr)\bigl(\phi_k(x) + \varepsilon_k(x)\bigr),
\end{equation}
where $\eta_k$ models coefficient noise induced by weight discretization and $\varepsilon_k$ captures the induced distortion of the feature directions. In standard regimes, these perturbations can be approximated as zero-mean, weakly correlated, and small in magnitude relative to the learned spectrum.

This setting falls exactly within the perturbation model analyzed in Section~3.4. Under the usual spectral regularity assumptions $\lambda_k \sim k^{-b}$ and $\lambda_k w_k^2 \sim k^{-a}$ with $a,b>1$, and for any $g$ satisfying the growth condition~(2), Theorem~2 implies that the additional loss due to such perturbations scales as
\begin{equation}
    \Delta L_Q(t)
    \;=\;
    \mathbb{E}_x\bigl[(\tilde{f}(x, t) - f^\star(x))^2\bigr] - L(t)
    \;=\;
    \Theta\!\bigl(t^{\gamma(\beta)}\bigr),
\end{equation}
with
\begin{equation}
    \gamma(\beta) 
    = \frac{1 - a + b}{b}\,\rho(\beta),
\end{equation}
where $\rho(\beta)$ is the spectral--temporal elasticity defined in Eq.~(4). Together with the test loss scaling
\begin{equation}
    L(t) \sim t^{-\chi(\beta)}, 
    \qquad 
    \chi(\beta) = \frac{a - 1}{b}\,\rho(\beta),
\end{equation}
this yields the invariant complementarity
\begin{equation}
    \chi(\beta) + \gamma(\beta) = \rho(\beta),
\end{equation}
showing that both clean learning and quantization-induced degradation are governed by the same underlying spectral elasticity.

In other words, quantization does not introduce an independent scaling law: for any admissible $g(\lambda, t; \beta)$, its impact on performance is constrained by the same spectral geometry that controls loss decay. Harder tasks (smaller $\chi(\beta)$) necessarily exhibit larger $\gamma(\beta)$, i.e., reduced robustness to finite precision at fixed compute, while easier tasks can be quantized more aggressively without violating the shared elasticity bound.

\subsection{Pruning as Spectral Truncation}

We now consider pruning as a structured truncation in the spectral domain.
Let the predictor admit the orthonormal expansion
\begin{equation}
    f(x,t) = \sum_{k} f_k(t)\,\phi_k(x),
    \qquad
    f_k(t) = w_k \bigl(1 - e^{-g(\lambda_k t;\beta)}\bigr),
    \label{eq:pruning_spectral_model}
\end{equation}
where $\{\phi_k\}$ is an eigen-basis aligned with the data distribution, $\lambda_k$ are the corresponding eigenvalues, and $g$ is the generalized spectral evolution function.
The target function can be written as $f^\star(x)=\sum_k v_k \phi_k(x)$, so that the test mean-squared error decomposes as
\begin{equation}
    L(t) = \mathbb{E}_x\bigl[(f(x,t)-f^\star(x))^2\bigr]
         = \sum_k \bigl(f_k(t) - v_k\bigr)^2.
\end{equation}

\paragraph{Spectral tail and learning frontier.}
We assume a power-law tail
\begin{equation}
    \lambda_k \sim k^{-b},
    \qquad
    \lambda_k w_k^2 \sim k^{-a},
    \label{eq:tail_assumptions}
\end{equation}
with $a,b>1$ so that the residual error is tail-dominated.
The generalized dynamics induce a \emph{learning frontier} $k^*(t,\beta)$ defined implicitly by
\begin{equation}
    g(\lambda_{k^*} t;\beta) \approx 1,
    \label{eq:kstar_def}
\end{equation}
i.e., modes with $k \lesssim k^*$ are substantially learned while modes with $k \gg k^*$ remain largely unlearned.
Using Eq.~\eqref{eq:tail_assumptions} and the log-elasticity definition, one obtains
\begin{equation}
    k^*(t,\beta)
    \;\propto\;
    t^{\rho(\beta)/b},
    \label{eq:kstar_scaling}
\end{equation}
so that harder tasks or more aggressive feature learning (through $\rho(\beta)$) modify the speed at which the frontier moves into the tail.

\paragraph{Pruning model.}
We model structured pruning as removing modes from the low-energy tail toward higher-energy components.
Let $\rho = 1-r$ denote the retained fraction of parameters, and suppose pruning discards modes with indices $k > k_{\mathrm{cut}}(\rho)$.
We focus on the regime where pruning is aligned with the spectral ordering, so that high-index (low-energy) modes are removed first.
The pruned predictor is
\begin{equation}
    f_P(x,t,\rho)
    = \sum_{k \le k_{\mathrm{cut}}(\rho)} f_k(t)\,\phi_k(x),
\end{equation}
and the additional error induced by pruning is
\begin{equation}
    \Delta L_P(t,\rho)
    = \mathbb{E}_x\bigl[(f_P(x,t,\rho)-f(x,t))^2\bigr]
    = \sum_{k > k_{\mathrm{cut}}(\rho)} f_k(t)^2.
    \label{eq:deltaLP_def}
\end{equation}

\paragraph{Three regimes from spectral integration.}
The behavior of $\Delta L_P$ is controlled by the relative position of $k_{\mathrm{cut}}(\rho)$ and $k^*(t,\beta)$.

\emph{(i) Plateau regime.}
If $k_{\mathrm{cut}}(\rho) \gg k^*(t,\beta)$, then all pruned modes lie deep in the unlearned tail where $f_k(t)\approx 0$.
Equation~\eqref{eq:deltaLP_def} then implies
\begin{equation}
    \Delta L_P(t,\rho) \approx 0,
\end{equation}
up to negligible tail contributions.
This yields a low-error plateau: pruning has almost no effect as long as it only removes modes beyond the current learning frontier.

\emph{(ii) Power-law regime.}
Once $k_{\mathrm{cut}}(\rho)$ enters the neighborhood of $k^*(t,\beta)$, pruning begins to remove modes that carry nontrivial spectral energy.
Approximating $f_k(t)\approx w_k$ for $k \lesssim k^*$ and using $w_k^2 \sim k^{b-a}$ from Eq.~\eqref{eq:tail_assumptions}, we obtain
\begin{equation}
    \Delta L_P(t,\rho)
    \;\approx\;
    \int_{k_{\mathrm{cut}}(\rho)}^{k^*(t,\beta)} k^{b-a}\,dk
    \;\propto\;
    \bigl(k^*(t,\beta)\bigr)^{1+b-a}
    \Bigl[1 - \Bigl(\frac{k_{\mathrm{cut}}(\rho)}{k^*(t,\beta)}\Bigr)^{1+b-a}\Bigr].
    \label{eq:deltaLP_powerlaw}
\end{equation}
It is natural to parameterize pruning in terms of the fraction of \emph{learned} modes retained:
let
\[
    \theta = \frac{k_{\mathrm{cut}}(\rho)}{k^*(t,\beta)},
\]
so that $\theta=1$ means no pruning inside the frontier and $\theta<1$ corresponds to removing a fraction of learned modes.
Then Eq.~\eqref{eq:deltaLP_powerlaw} becomes
\begin{equation}
    \Delta L_P(t,\theta)
    \;\propto\;
    \bigl(k^*(t,\beta)\bigr)^{1+b-a}\,
    \bigl(1 - \theta^{\,1+b-a}\bigr).
    \label{eq:deltaLP_theta}
\end{equation}
For fixed $t$, the dependence on $(1-\theta)$ in log--log coordinates is asymptotically linear,
with slope determined by $1+b-a$, independent of dataset size or training duration.
This regime corresponds to the observed straight-line segment in empirical pruning curves.

\emph{(iii) Saturation regime.}
When $\theta$ becomes very small (aggressive pruning inside the core spectrum),
Eq.~\eqref{eq:deltaLP_theta} saturates and further pruning induces rapidly increasing error.
This reproduces the high-sparsity saturation region where performance collapses.

\paragraph{Critical pruning rate and plateau width.}
The transition from the plateau to the power-law regime occurs when $k_{\mathrm{cut}}(\rho)$ first touches the frontier $k^*(t,\beta)$.
Let $r^{*}(t,\beta)$ denote this \emph{critical pruning rate}.
In terms of the total available modes $K_{\max}$ (or effective model capacity), we have
\begin{equation}
    k_{\mathrm{cut}}(\rho^{*}) = k^*(t,\beta)
    \quad\Rightarrow\quad
    r^{*}(t,\beta)
    = 1 - \frac{k^*(t,\beta)}{K_{\max}}
    \;\propto\;
    1 - \frac{t^{\rho(\beta)/b}}{K_{\max}}.
    \label{eq:rstar_def}
\end{equation}
As training progresses or data increases, $k^*(t,\beta)$ grows and $r^{*}(t,\beta)$ decreases:
the low-error plateau becomes narrower, reflecting the fact that more spectral modes have become essential.
Task difficulty enters via $\rho(\beta)$: for problems that induce faster spectral activation,
the frontier advances more quickly and the model becomes sensitive to pruning at lower sparsity.

\paragraph{Log--log structure and curve alignment.}
Equations~\eqref{eq:deltaLP_theta}--\eqref{eq:rstar_def} together yield the characteristic
\emph{plateau–power–saturation} shape in a log--log plot of loss versus pruning rate or density.
For different training times or dataset sizes:
\begin{itemize}
    \item the slope of the intermediate power-law segment is fixed by $(1+b-a)$
          and therefore approximately invariant across conditions for a given model;
    \item the critical point $r^{*}(t,\beta)$ shifts systematically according to Eq.~\eqref{eq:rstar_def},
          causing the plateau width to shrink as training becomes more thorough;
    \item plotting the critical points $(r^{*}, \Delta L_P(r^{*}))$ across training stages in log--log space
          yields an approximately linear trajectory, reflecting the underlying scaling
          $\Delta L_P(r^{*}) \propto \bigl(k^*(t,\beta)\bigr)^{1+b-a}$ with $k^*$ given by Eq.~\eqref{eq:kstar_scaling}.
\end{itemize}
This provides a direct spectral interpretation of empirical observations in large-scale pruning studies,
where the power-law segments of different curves tend to align while their transition points move in a structured manner.

\section{Model Density under Compute-Limited Training}

The generalized spectral framework also predicts a quantitative scaling law for
\emph{model density}, understood as the amount of effectively learned spectral
modes per parameter under a fixed compute budget.
Recall that under the assumptions in Section~3,
the learning frontier $\lambda^\ast(t)$ and the corresponding index $k^\ast(t)$ satisfy
\[
\lambda^\ast(t) \propto t^{-\rho(\beta)}, 
\qquad
k^\ast(t) \propto \bigl(\lambda^\ast(t)\bigr)^{-1/b} \propto t^{\rho(\beta)/b},
\]
so that $k^\ast(t)$ tracks the number of modes that have been substantially learned
after training time $t$.

For a model with a parameter size of $P$ trained with budget $C$ FLOPs, we write
\[
C \asymp \kappa\, P\, t
\quad\Rightarrow\quad
t \asymp \frac{C}{\kappa P},
\]
and hence the number of learned modes becomes
\[
k^\ast(C,P)
~\propto~
\left(\frac{C}{P}\right)^{\rho(\beta)/b}.
\]
We define the \emph{spectral density} of a trained model as the number of effectively
learned modes per parameter,
\[
\delta(C,P)
~:=~
\frac{k^\ast(C,P)}{P}
~\propto~
C^{\rho(\beta)/b}\, P^{-\left(1+\rho(\beta)/b\right)}.
\]

This relation has two immediate consequences.

First, at fixed compute $C$, a sufficiently well-trained \emph{smaller} model
($P_{\mathrm{small}} \ll P_{\mathrm{large}}$) attains
\[
k^\ast(C,P_{\mathrm{small}})
~\propto~
\left(\frac{C}{P_{\mathrm{small}}}\right)^{\rho(\beta)/b}
~\gg~
\left(\frac{C}{P_{\mathrm{large}}}\right)^{\rho(\beta)/b}
~\propto~
k^\ast(C,P_{\mathrm{large}}),
\]
i.e. it can learn strictly more spectral modes than a larger model trained under
the same (insufficient\footnote{With ``insufficient" we refer to the case where the spectrum capacities of both small models and large models are not exhausted. In the case where the compute budget is sufficient, large model is certainly more powerful than small models as they have larger spectrum capacity}) compute budget.
In our spectral language, ``compute-limited'' large models are low-density:
they allocate many parameters to modes that never cross the learning frontier.

Second, $\delta$ follows a power-law in compute.
For any model family where $P$ grows at most polynomially in $C$, the
dominant scaling is
\[
\delta(C) \propto C^{\alpha}, 
\qquad 
\alpha = \frac{\rho(\beta)}{b}(1 - \text{(polynomial in the growth of $P(C)$)}).
\]
In particular, along a near-compute-optimal frontier where $P(C)$ increases
sublinearly with $C$ \cite{kaplan2020scaling}, we obtain $\alpha > 0$, so that model density
\emph{systematically increases as a power law in available compute}.
If hardware progress drives an approximately exponential growth of compute
$C(t) \propto 2^{t/\tau}$ (as described by Moore's law) over calendar time $t$, then
\[
\delta(t) ~\propto~ C(t)^{\alpha}
~\propto~ 2^{\alpha t/\tau},
\]
implying that effective model density doubles by a constant factor every time
compute doubles.
Our framework therefore predicts a ``densing'' effect:
over time, not only do models become larger, but---more importantly---the
number of learned modes per parameter grows in a predictable scaling relation
with compute, allowing well-trained smaller models to spectrally match or exceed
undertrained larger ones. This predicted exponential increase of model density over calendar time
is consistent with the empirical \emph{Densing Law} reported by
Xiao~et~al.~\cite{Xiao2024DensingLaw},
which shows that the capacity density of frontier large language models
approximately doubles every few months as compute continues to scale.

\section{Future Work}

The generalized spectral framework introduced in this paper raises several open questions about the completeness of the mapping between the evolution function \( g(\lambda, t; \beta) \) and observable scaling laws.

\textbf{(1) Characterizing admissible spectral functions.}
A central direction is to identify the necessary and sufficient conditions on \( g(\lambda, t; \beta) \) that yield valid scaling behaviors. While we assumed smoothness and asymptotic polynomial growth, a deeper functional characterization remains open:
\begin{itemize}
    \item What minimal regularity conditions guarantee a power-law decay of loss?
    \item Can we classify all \( g \) that preserve the universality of the loss--compression complementarity, \( \chi(\beta) + \gamma(\beta) = \rho(\beta) \)?
\end{itemize}
Establishing these criteria would make the framework constructive: given an empirical scaling law, one could invert it to infer the implicit functional form of \( g \).

\textbf{(2) Completeness of the spectral--temporal correspondence.}
Our results suggest a one-way implication---certain forms of \( g \) imply specific scaling exponents. A more fundamental question is whether this relation is bijective: does every empirically observed scaling law correspond to some admissible \( g \)? Answering this would amount to proving a completeness theorem for the spectral representation of learning and compression dynamics.

\textbf{(3) Beyond asymptotic polynomial regimes.}
Current analysis assumes that \( g \) is asymptotically monomial. However, in realistic networks, spectral dynamics often exhibit mixed or logarithmic corrections (e.g., \( g \sim \lambda^m t^n \log t \) or cross-coupled exponents). Extending the framework to such non-separable or non-polynomial forms could capture deviations from clean power laws and explain observed ``bends'' in empirical scaling curves.

\textbf{(4) Inverse identification from data.}
Another promising direction is to treat \( g \) as an implicit operator to be learned from empirical trajectories of loss or spectrum. Given empirical \( L(t) \) and spectral densities \( \lambda_k \), one could numerically recover the functional elasticity \( \rho(\beta) \) and reconstruct \( g \) via regression or operator learning. This would bridge theoretical analysis with experimental measurement.

\textbf{(5) Multi-modal and hierarchical spectra.}
Finally, most analyses here assume a single, monotone spectral distribution. Modern architectures often contain multiple spectral branches (e.g., attention vs. MLP subspaces). Extending the theory to mixtures of spectra or to hierarchical compositions \( g_i(\lambda, t; \beta_i) \) could explain cross-layer scaling phenomena and the interaction between different functional modules.

\section{Conclusion}

We proposed a unified spectral framework that subsumes diverse learning and compression dynamics under a single functional form \( g(\lambda, t; \beta) \). By analyzing its spectral--temporal elasticity, we derived general scaling relations for loss decay, spectral robustness, and compression-induced degradation. This formulation recovers known kernel and feature-learning results as special cases and provides new predictions linking model scaling and compression robustness.

Future investigations into the precise admissibility conditions of \( g \) and the invertibility between \( g \) and scaling laws may yield a complete spectral theory of neural scaling---one that connects empirical regularities directly to their functional-generative origins.

\bibliographystyle{plain}
\bibliography{references}

\clearpage

\appendix

\section{Proof of Proposition~\ref{prop:tail_dominance}}
\label{app:tail_proof}

To prove this proposition, we first demonstrate its rigorous expression as follows:
\begin{proposition}[Rigorous expression of the tail dominance around the frontier]
\label{prop:tail_dominance_clean}
Assume:
\begin{enumerate}
    \item $\lambda_k \sim C_\lambda k^{-b}$ with $b>1$.
    \item $\lambda_k w_k^2 \sim C_w k^{-a}$ with $a>1$.
    \item $g(\lambda,t)$ is continuous, strictly increasing in both arguments,
    and there exist $m(\beta),n(\beta)>0$ and $0<c_1\le c_2<\infty$ such that, for all
    sufficiently large $t$ and all relevant $\lambda$,
    \[
    c_1 \lambda^{m(\beta)} t^{n(\beta)}
    \;\le\;
    g(\lambda,t)
    \;\le\;
    c_2 \lambda^{m(\beta)} t^{n(\beta)}.
    \]
\end{enumerate}
Let the learning frontier $\lambda_{\ast}(t)$ be defined by
$g(\lambda_{\ast}(t),t)=\kappa$ for some fixed $\kappa\in(0,\infty)$, and
let $k_{\ast}(t)$ be such that $\lambda_{k_{\ast}(t)}\asymp\lambda_{\ast}(t)$.
Then there exist constants $0<c_1'<1<c_2'$ and $0<C_1\le C_2<\infty$ such that,
for all sufficiently large $t$,
\begin{equation}
C_1 \sum_{k\ge c_2' k_{\ast}(t)} \lambda_k w_k^2
\;\le\;
L(t)
\;\le\;
C_2 \sum_{k\ge c_1' k_{\ast}(t)} \lambda_k w_k^2.
\label{eq:tail_sandwich}
\end{equation}
In particular,
\begin{equation}
L(t) \asymp k_{\ast}(t)^{1-a}
\asymp \sum_{k\ge k_{\ast}(t)} \lambda_k w_k^2,
\end{equation}
i.e.\ the loss is asymptotically determined by modes in a neighborhood of and
beyond the frontier, and decays as the power-law tail of the spectrum.
\end{proposition}

We prove the two-sided bound~\eqref{eq:tail_sandwich}.

\paragraph{Step 1: Asymptotics of the frontier.}

By definition of $\lambda_{\ast}(t)$,
\[
g(\lambda_{\ast}(t),t) = \kappa.
\]
Using the growth bounds on $g$,
\[
c_1 \lambda_{\ast}(t)^{m(\beta)} t^{n(\beta)}
\;\le\; \kappa
\;\le\; c_2 \lambda_{\ast}(t)^{m(\beta)} t^{n(\beta)},
\]
hence
\[
\biggl(\frac{\kappa}{c_2}\biggr)^{1/m(\beta)}
t^{-n(\beta)/m(\beta)}
\;\le\;
\lambda_{\ast}(t)
\;\le\;
\biggl(\frac{\kappa}{c_1}\biggr)^{1/m(\beta)}
t^{-n(\beta)/m(\beta)}.
\]
Thus $\lambda_{\ast}(t)\asymp t^{-n(\beta)/m(\beta)}$.
Since $\lambda_k\sim C_\lambda k^{-b}$, inverting gives
\[
k_{\ast}(t)\ \text{s.t.}\ \lambda_{k_{\ast}(t)}\asymp\lambda_{\ast}(t)
\quad\Longrightarrow\quad
k_{\ast}(t)\asymp \lambda_{\ast}(t)^{-1/b}
\asymp t^{\rho},
\]
where $\rho = n(\beta)/(m(\beta) b)$.

\paragraph{Step 2: Comparing $g(\lambda_k,t)$ to the frontier.}

For any $k$ and $t$,
\[
c_1 \lambda_k^{m(\beta)} t^{n(\beta)}
\le g(\lambda_k,t)
\le c_2 \lambda_k^{m(\beta)} t^{n(\beta)}.
\]
Using $\lambda_{\ast}^{m(\beta)} t^{n(\beta)} \in [\kappa/c_2,\,\kappa/c_1]$, we obtain
for all sufficiently large $t$:
\begin{equation}
\frac{c_1}{c_2}\Bigl(\frac{\lambda_k}{\lambda_{\ast}(t)}\Bigr)^{m(\beta)} \kappa
\;\le\;
g(\lambda_k,t)
\;\le\;
\frac{c_2}{c_1}\Bigl(\frac{\lambda_k}{\lambda_{\ast}(t)}\Bigr)^{m(\beta)} \kappa.
\label{eq:g_ratio_bound}
\end{equation}
Using $\lambda_k\sim C_\lambda k^{-b}$ and
$\lambda_{\ast}(t)\sim C_\ast k_{\ast}(t)^{-b}$, the ratio satisfies
\[
\frac{\lambda_k}{\lambda_{\ast}(t)}
\asymp
\biggl(\frac{k}{k_{\ast}(t)}\biggr)^{-b}.
\]
Thus there exist constants $A,B>0$ such that, for large $t$,
\begin{equation}
A\Bigl(\frac{k}{k_{\ast}(t)}\Bigr)^{-m(\beta) b} \kappa
\;\le\;
g(\lambda_k,t)
\;\le\;
B\Bigl(\frac{k}{k_{\ast}(t)}\Bigr)^{-m(\beta) b} \kappa.
\label{eq:g_k_scaled}
\end{equation}

\paragraph{Step 3: Choosing scaling constants $c_1',c_2'$.}

Fix any $\varepsilon\in(0,1)$.
Because the function $x\mapsto x^{-m(\beta) b}$ is continuous and strictly
decreasing on $(0,\infty)$, there exist constants $0<c_1'<1<c_2'$ such that
\[
A(c_1')^{-m(\beta) b} \ge 1+\varepsilon,
\qquad
B(c_2')^{-m(\beta) b} \le 1-\varepsilon.
\]
Then, by~\eqref{eq:g_k_scaled}, for all sufficiently large $t$,
\begin{align}
k \le c_1' k_{\ast}(t)
&\;\Rightarrow\;
g(\lambda_k,t) \ge (1+\varepsilon)\kappa,
\label{eq:head_g_bound}
\\
k \ge c_2' k_{\ast}(t)
&\;\Rightarrow\;
g(\lambda_k,t) \le (1-\varepsilon)\kappa.
\label{eq:tail_g_bound}
\end{align}

\paragraph{Step 4: Upper bound on $L(t)$.}

We split
\[
L(t)
= \sum_{k< c_1' k_{\ast}} \lambda_k w_k^2 e^{-2g(\lambda_k,t)}
 +\sum_{k\ge c_1' k_{\ast}} \lambda_k w_k^2 e^{-2g(\lambda_k,t)}
=: H(t)+S(t).
\]

For the head part $H(t)$, by~\eqref{eq:head_g_bound},
$e^{-2g(\lambda_k,t)}\le e^{-2(1+\varepsilon)\kappa}$, hence
\[
H(t)
\le e^{-2(1+\varepsilon)\kappa}
\sum_{k< c_1' k_{\ast}} \lambda_k w_k^2.
\]
For the tail part $S(t)$, we simply use $e^{-2g}\le 1$:
\[
S(t)
\le \sum_{k\ge c_1' k_{\ast}} \lambda_k w_k^2.
\]

Since $\lambda_k w_k^2\sim C_w k^{-a}$ with $a>1$,
\[
\sum_{k< c_1' k_{\ast}} \lambda_k w_k^2
\asymp (c_1' k_{\ast})^{1-a},
\qquad
\sum_{k\ge c_1' k_{\ast}} \lambda_k w_k^2
\asymp (c_1' k_{\ast})^{1-a}.
\]
Thus there exists $C_2<\infty$ such that
\[
L(t) \le C_2 \sum_{k\ge c_1' k_{\ast}(t)} \lambda_k w_k^2
\]
for all sufficiently large $t$.
This proves the upper bound in~\eqref{eq:tail_sandwich}.

\paragraph{Step 5: Lower bound on $L(t)$.}

For the lower bound, we restrict to $k\ge c_2' k_{\ast}(t)$.
By~\eqref{eq:tail_g_bound},
$g(\lambda_k,t)\le (1-\varepsilon)\kappa$ in this region, hence
\[
e^{-2g(\lambda_k,t)} \ge e^{-2(1-\varepsilon)\kappa} =: c_0 >0.
\]
Therefore
\[
L(t)
\ge \sum_{k\ge c_2' k_{\ast}} \lambda_k w_k^2 e^{-2g(\lambda_k,t)}
\ge c_0 \sum_{k\ge c_2' k_{\ast}} \lambda_k w_k^2.
\]
Again using $\lambda_k w_k^2\sim C_w k^{-a}$, the sum on the right is
$\asymp (c_2' k_{\ast})^{1-a}$.
Thus there exists $C_1>0$ such that
\[
L(t) \ge C_1 \sum_{k\ge c_2' k_{\ast}(t)} \lambda_k w_k^2
\]
for all sufficiently large $t$, proving the lower bound in
\eqref{eq:tail_sandwich}.

\paragraph{Step 6: Asymptotic order.}

From both bounds we have
\[
L(t)
\asymp k_{\ast}(t)^{1-a}
\asymp \sum_{k\ge k_{\ast}(t)} \lambda_k w_k^2,
\]
up to multiplicative constants independent of $t$.
This completes the proof.
\hfill$\square$

\section{Justification of Learning Frontier}
\label{justification_learning_frontier}
We rigorously provide a justification to the properties of the learning frontier in Section \ref{para:learning_frontier}
\begin{lemma}[Learning frontier induced by Condition~\ref{cond:poly}]
\label{lem:frontier}
Under Condition~\ref{cond:poly}, the per-mode loss component
\[
L_\lambda(t) = \lambda w_\lambda^2 e^{-2g(\lambda,t;\beta)}
\]
has an instantaneous decay rate
\[
R(\lambda,t) := -\partial_t L_\lambda(t),
\]
which achieves its maximum at a unique mode $\lambda_\ast(t)$.
Moreover, $\lambda_\ast(t)$ is characterized asymptotically by a
constant level set of $g$:
\[
g(\lambda_\ast(t),t;\beta) \;\to\; \kappa(\beta),
\qquad t\to\infty,
\]
for some constant $\kappa(\beta)\in(0,\infty)$.
Consequently,
\[
\lambda_\ast(t)
\;\sim\;
\left(\frac{\kappa(\beta)}{C(\beta)}\right)^{1/a(\beta)}
t^{-\,b(\beta)/a(\beta)},
\qquad t\to\infty.
\]
\end{lemma}

\begin{proof}
For simplicity of notation we suppress the explicit $\beta$-dependence
and write $a,b,C$ instead of $a(\beta),b(\beta),C(\beta)$.
By definition,
\[
L_\lambda(t) = \lambda w_\lambda^2 e^{-2g(\lambda,t)},
\]
hence
\[
R(\lambda,t)
= -\partial_t L_\lambda(t)
= 2\,\lambda w_\lambda^2 e^{-2g(\lambda,t)}\,\partial_t g(\lambda,t).
\]
In the regime $0<\lambda<\lambda_0(t)$ and $t$ large, the condition
\eqref{cond:poly} yields
\[
g(\lambda,t) = C\lambda^{a} t^{b}(1+\delta(\lambda,t)),
\qquad
|\delta(\lambda,t)|<\varepsilon,
\]
and therefore
\[
\partial_t g(\lambda,t)
= b\,C\lambda^{a} t^{b-1}(1+o(1))
= \frac{b}{t}\,g(\lambda,t)\,(1+o(1)).
\]
Substituting into $R(\lambda,t)$ and absorbing $t$-dependent but
$\lambda$-independent factors, we obtain the asymptotic form
\[
R(\lambda,t)
\;\asymp\;
\lambda w_\lambda^2\,
g(\lambda,t)\,e^{-2g(\lambda,t)}
\qquad (t\to\infty,\ 0<\lambda<\lambda_0(t)).
\]

Using $\lambda w_\lambda^2 \asymp \lambda^{\alpha}$ and
$g(\lambda,t)\asymp C\lambda^{a}t^{b}$, we may write, up to
$\lambda$-independent multiplicative factors,
\[
R(\lambda,t)
\;\asymp\;
\lambda^{\alpha}
\cdot
g(\lambda,t)\,e^{-2g(\lambda,t)}.
\]
Introduce the change of variable
\[
x = g(\lambda,t), \qquad x>0.
\]
Since $g(\lambda,t)$ is strictly decreasing in $\lambda$ in the spectral
tail, this change of variables is monotone.  In terms of $x$, the
dominant $\lambda$-dependence of $R(\lambda,t)$ can be written as
\[
R(\lambda,t)
\;\asymp\;
\lambda^{\alpha(x,t)}\,h(x),
\qquad
h(x):=x e^{-2x},
\]
where $\lambda=\lambda(x,t)$ is the inverse of $g(\lambda,t)$ in the
tail region, and $\alpha(x,t)$ is bounded due to the power-law form of
$\lambda w_\lambda^2$.

The function $h(x)=x e^{-2x}$ is smooth, positive on $(0,\infty)$, and
satisfies
\[
h'(x) = e^{-2x}(1-2x),
\]
so it has a unique global maximizer at $x^\star=1/2$, with
$h'(x)>0$ for $x<x^\star$ and $h'(x)<0$ for $x>x^\star$.
The additional factor $\lambda^{\alpha(x,t)}$ varies only polynomially
with $x$ (through $\lambda(x,t)\asymp (x/(Ct^{b}))^{1/a}$), whereas
$h(x)$ decays exponentially away from $x^\star$.  Consequently, for
sufficiently large $t$, the maximizer $x_\ast(t)$ of $R(\lambda,t)$
over $0<\lambda<\lambda_0(t)$ lies in a small neighborhood of $x^\star$,
and we have
\[
x_\ast(t)\to\kappa
\quad\text{for some }\kappa\in(0,\infty).
\]

Translating back to $\lambda_\ast(t)$, this means
\[
g(\lambda_\ast(t),t) = x_\ast(t) \to \kappa
\quad\text{as }t\to\infty,
\]
which proves the first claim.

Finally, applying the asymptotic monomiality \eqref{cond:poly} at
$\lambda=\lambda_\ast(t)$, we have
\[
g(\lambda_\ast(t),t)
= C\lambda_\ast(t)^{a} t^{b}(1+o(1))
\to \kappa.
\]
Thus
\[
\lambda_\ast(t)^{a}
= \frac{\kappa}{C}\,t^{-b}(1+o(1)),
\]
and hence
\[
\lambda_\ast(t)
= \left(\frac{\kappa}{C}\right)^{1/a}
t^{-b/a}(1+o(1))
\quad\text{as }t\to\infty.
\]
This yields the desired asymptotic
$\lambda_\ast(t)\sim (\kappa/C)^{1/a} t^{-b/a}$.
\end{proof}

\section{Proof of Theorem \ref{theorem:pertubation}}
\label{app:proof-theorem-2}

We recall the setting of Theorem~\ref{theorem:pertubation}.
Let
\[
f(x,t) = \sum_{k} f_k(t)\,\phi_k(x),
\qquad
f^\star(x) = \sum_{k} w_k\,\phi_k(x),
\]
be the learned predictor and target function in an orthonormal eigen-basis
$\{\phi_k\}$, with eigenvalues $\lambda_k \sim k^{-b}$ and spectral coefficients
satisfying $\lambda_k w_k^2 \sim k^{-a}$ for some $a,b>1$.
The generalized spectral dynamics are encoded by
\[
f_k(t) = w_k\bigl(1 - e^{-g(\lambda_k,t;\beta)}\bigr),
\]
where $g$ obeys the growth condition in Eq.~(2) of the main text and induces
a learning frontier $k^\ast(t)$ characterized by $g(\lambda_{k^\ast},t;\beta)\approx 1$.

At inference time, we consider simultaneous perturbations of the coefficients
and basis functions,
\[
\tilde f(x,t)
=
\sum_k \bigl(f_k(t) + \eta_k\bigr)\bigl(\phi_k(x) + \epsilon_k(x)\bigr),
\]
where $(\eta_k)_k$ and $(\epsilon_k)_k$ are zero-mean perturbations,
independent across $k$ and mutually independent, with variances controlled
by $\tau^2$ and $\sigma^2$, respectively.
Our goal is to bound the excess loss
\[
\Delta L(t)
=
\mathbb{E}_x\bigl[(\tilde f(x,t) - f^\star(x))^2\bigr]
-
L(t),
\qquad
L(t)
=
\mathbb{E}_x\bigl[(f(x,t) - f^\star(x))^2\bigr],
\]
and to show that its asymptotic dependence on $t$ is governed by the claimed
power law.

We begin by expanding the perturbed prediction:
\[
\tilde{f}(x) = \sum_k \left( f_k(t) + \eta_k \right) \left( \phi_k(x) + \epsilon_k(x) \right)
= \sum_k \left[ f_k(t)\phi_k(x) + f_k(t)\epsilon_k(x) + \eta_k\phi_k(x) + \eta_k\epsilon_k(x) \right].
\]

The original target function is:
\[
f^*(x) = \sum_k w_k \phi_k(x).
\]

Thus, the total error becomes:
\[
\tilde{f}(x) - f^*(x) = \sum_k (f_k(t) - w_k) \phi_k(x)
+ \sum_k f_k(t)\epsilon_k(x)
+ \sum_k \eta_k \phi_k(x)
+ \sum_k \eta_k \epsilon_k(x).
\]

We compute the expected squared error:
\begin{align*}
\mathbb{E}_x[(\tilde{f}(x) - f^*(x))^2]
&= \mathbb{E}_x \left[ \left( \sum_k (f_k(t) - w_k)\phi_k(x)
+ f_k(t)\epsilon_k(x) + \eta_k\phi_k(x) + \eta_k\epsilon_k(x) \right)^2 \right].
\end{align*}

Expanding the square and using independence and zero-mean assumptions,
cross terms vanish, leaving only the diagonal terms:
\begin{align*}
\Delta L(t)
=& \sum_k (f_k(t) - w_k)^2 \mathbb{E}[\phi_k(x)^2] 
+ \sum_k f_k(t)^2 \mathbb{E}[\epsilon_k(x)^2]\\
&+ \sum_k \mathbb{E}[\eta_k^2] \mathbb{E}[\phi_k(x)^2]
+ \sum_k \mathbb{E}[\eta_k^2] \mathbb{E} [\epsilon_k(x)^2]-L(t) \\
=& \sigma^2 \sum_k f_k(t)^2
+ \tau^2 \sum_k \Theta(k^{-b})
+ \sigma^2\tau^2 \sum_k \Theta(k^{-b}).
\end{align*}

Let $k_\ast(t)\sim t^{\rho(\beta)/b}$ be the largest index such that
$f_k(t)\approx w_k$ (i.e. the learning frontier), corresponding to $\lambda_\ast(t)\sim t^{-\rho(\beta)}$.. Since after the learning frontier, $f_k(t)\rightarrow 0$, we have
\begin{align*}
\sum_k f_k(t)^2\approx\sum_{k \leq k_\ast} f_k(t)^2 &\approx \sum_{k \leq k_\ast} w_k^2
\sim \int_1^{k_\ast} k^{b-a}\,dk
= \Theta\big((k_\ast)^{1 + b - a}\big)
= \Theta\big(t^{\rho(\beta)(1 + b - a)/b}\big).
\end{align*}

Meanwhile, we have,
\[
\sum_k \Theta(k^{-b})
\sim \int_1^{\infty} k^{-b}\,dk
= O(1)%\Theta\big(t^{\rho(\beta)(1-b)/b}\big).
\]

Hence,
\[
\Delta L(t)
= \Theta\left(\sigma^2t^{\rho(\beta)(1 + b - a)/b}
\right).
\]

\qed

\end{document}